\documentclass[letterpaper, 10pt, journal]{ieeetran}
\usepackage[colorlinks=true,
    linkcolor=blue,     % internal links
    citecolor=teal,     % citations
    urlcolor=magenta    % external links
]{hyperref}
\usepackage{graphicx} % Required for inserting images
\usepackage{amsmath}
 \usepackage{amsthm}
\usepackage{mathtools}
\usepackage{amssymb}
\usepackage{booktabs}
\usepackage{xcolor}
\usepackage{algorithm}
\usepackage{algpseudocode}
\usepackage[font=small]{caption}
\usepackage{cite}
\usepackage[capitalize]{cleveref}

\newenvironment{talign}
  {\let\displaystyle\textstyle\align}
  {\endalign}
\newenvironment{talign*}
  {\begingroup\let\displaystyle\textstyle
   \begin{equation*}\begin{aligned}}
  {\end{aligned}\end{equation*}\endgroup}
\newcommand{\cA}{\mathcal{A}}
\newcommand{\cS}{\mathcal{S}}
\newcommand{\cM}{\mathcal{M}}
\newcommand{\cF}{\mathcal{F}}

\newcommand{\KL}{\textup{KL}}

\newcommand{\bR}{\mathbb{R}}

\newcommand{\oA}{\overline{A}}
\newcommand{\oQ}{\overline{Q}}

\newcommand{\Proj}{\textup{Proj}}
\newcommand{\hpi}{\widehat{\pi}}

\DeclareMathOperator*{\argmax}{argmax}
\newcommand{\bE}{\mathbb{E}}

\theoremstyle{definition}

\newtheorem{theorem}{Theorem}
\newtheorem{lemma}{Lemma}
\newtheorem{rmk}{Remark}
\newtheorem{defi}{Definition}
\newtheorem{example}{Example}
\allowdisplaybreaks

\begin{document}
\title{Optimism as Risk-Seeking in Multi-Agent Reinforcement Learning}
\author{Runyu Zhang$^{1}$, Na Li$^{2}$, Asuman Ozdaglar$^{1}$, Jeff Shamma$^{3}$, Gioele Zardini$^{1}$
\thanks{$^{1}$Runyu Zhang, Asuman Ozdaglar, and Gioele Zardini are with the Laboratory for Information \& Decision Systems,  Massachusetts Institute of Technology, {\{runyuzha, asuman, gzardini\}@mit.edu}}
\thanks{$^2$Na Li is with Harvard University, {nali@seas.harvard.edu}}
\thanks{$^3$Jeff Shamma is with the University of Illinois at Urbana-Champaign, {jshamma@illinois.edu}.}
}

\maketitle
\vspace{-10pt}
\begin{abstract}
Risk sensitivity has become a central theme in reinforcement learning (RL), where convex risk measures and robust formulations provide principled ways to model preferences beyond expected return. 
Recent extensions to multi-agent RL (MARL) have largely emphasized the risk-averse setting, prioritizing robustness to uncertainty.
In cooperative MARL, however, such conservatism often leads to suboptimal equilibria, and a parallel line of work has shown that \emph{optimism} can promote cooperation.
Existing optimistic methods, though effective in practice, are typically heuristic and lack theoretical grounding.
Building on the dual representation for convex risk measures, we propose a principled framework that interprets \emph{risk-seeking objectives} as optimism.
We introduce \emph{optimistic value functions}, which formalize optimism as divergence-penalized risk-seeking evaluations.
Building on this foundation, we derive a policy-gradient theorem for optimistic value functions, including explicit formulas for the entropic risk/KL-penalty setting, and develop decentralized optimistic actor-critic algorithms that implement these updates.
Empirical results on cooperative benchmarks demonstrate that risk-seeking optimism consistently improves coordination over both risk-neutral baselines and heuristic optimistic methods.
Our framework thus unifies risk-sensitive learning and optimism, offering a theoretically grounded and practically effective approach to cooperation in MARL.
\end{abstract}

%         -30 & Unif([0,14])&6\\
%         0 & 0&5
%     \end{tabular}
%     \caption{Caption}
%     \label{tab:my_label}
% \end{table}\
\section{Introduction}
Decision-making under uncertainty is central to reinforcement learning (RL). 
Standard RL optimizes expected return and is therefore \emph{risk-neutral} -- an assumption that is often too restrictive in risk-critical domains.
To overcome this, a rich body of work has studied \emph{risk-sensitive} and \emph{robust} RL.
Risk-sensitive RL leverages convex risk measures~\cite{follmer2002convex} such as entropic risk and CVar to shape Bellman updates and encode preferences over full return distributions~\cite{ruszczynski2010risk,chow2015risk,fei2021exponential,huang2021convergence}, while robust RL models distributional uncertainty via ambiguity sets of divergence-regularized optimization~\cite{ iyengar2005robust,nilim2005robust, panaganti2022sample,shi2022distributionally,shi2023curious}. 
These two perspectives are closely linked: risk-sensitive and robust Markov Decision Processes (MDPs) can be shown equivalent under dual formulations~\cite{ruszczynski2010risk, bauerle2022distributionally,osogami2012robustness,zhang2023softrobust}, yielding both theoretical guarantees and practical policy-gradient algorithms.

In multi-agent RL (MARL), risk-sensitive and robust formulations have also been explored, often addressing not only environment uncertainty but also unpredictable behavior of other agents~\cite{zhang20robust,shi2024breaking,lanzetti2025strategicallyrobustgametheory,mazumdar2025tractable}. 
Most of these works, however, adopt a \emph{risk-averse} stance: prioritizing safety at the expense of cooperation.
In cooperative tasks, this conservatism can prevent agents from reaching optimal joint solutions (c.f. \cite{wiegand2004analysis,matignon2007hysteretic, bellemare2023distributional,zhao2023optimistic}).
This limitation is exemplified by the problem of relative overgeneralization (RO)~\cite{wiegand2004analysis}, where agents converge to suboptimal equilibria by adapting too cautiously to noisy or misaligned partner behavior.
To overcome RO, a parallel line of research has explored \emph{optimism} or \emph{risk-seeking} as a driver of cooperation.
Early approaches such as distributed~$Q$-learning~\cite{Lauer00distributedQ}, hysteretic~$Q$-learning~\cite{matignon2007hysteretic}, and lenient learning~\cite{panait2006lenient,Wei16Lenient} softened negative updates to preserve cooperative actions.
Later methods such as FMQ learning~\cite{kapetanakis2002fmq} and optimistic exploration~\cite{zhang2025optimistic, zhao2023conditionally} reinforced this idea.
More recently, \emph{optimistic policy-gradient} algorithms (e.g., optimistic PPO variants) have been proposed~\cite{zhao2023optimistic}, biasing updates toward higher-value estimates to avoid being overly pessimistic about joint actions early on, thereby mitigating relative overgeneralization while preserving optimality at convergence.

Despite their empirical success, most of these optimistic MARL methods are heuristic.
Optimism is typically introduced through asymmetric learning rates~\cite{matignon2007hysteretic}, ad hoc return reshaping~\cite{kapetanakis2002fmq}, or tailored initialization strategies.
What is missing is a \emph{principled framework} that explains optimism mathematically and connects it to established concepts in risk-sensitive learning.
This motivates our central question: \emph{Can risk-seeking objectives provide a rigorous mechanism for optimism in MARL, thereby promoting cooperation in a theoretically grounded way?}

\paragraph*{Statement of contribution}
This work establishes a principled foundation for optimism in MARL by introducing optimistic value functions derived from convex risk measures.
We show that divergence-penalized dual formulations naturally instantiate optimism as a controlled form of risk-seeking, unifying heuristic optimistic updates with risk-sensitive RL theory.
Building on this framework, we develop a policy-gradient theorem for optimistic value functions, including explicit sample-based formulas in the entropic risk/KL-penalty setting.
We then design decentralized optimistic actor-critic algorithms that realize these updates in practice and validate them empirically on cooperative benchmarks, where they achieve more reliable coordination than both risk-neutral baselines and existing heuristic optimistic methods.

%In this work, we provide a rigorous mathematical foundation for optimism in multi-agent reinforcement learning by introducing \emph{optimistic value functions} grounded in convex risk measures. We show that divergence-penalized dual formulations naturally give rise to optimism as a form of controlled risk-seeking. Building on this foundation, we derive a policy-gradient theorem for optimistic value functions, including explicit sample-based formulas in the entropic risk/KL-penalty setting. Finally, we design decentralized optimistic actor-critic algorithms informed by this theory and validate them empirically on cooperative benchmarks, demonstrating improved coordination over both risk-neutral baselines and heuristic optimistic methods.

% \noindent\textbf{Our Contributions.}
% TODO: polish
% \begin{itemize}
%   \item \textbf{Rigorous mathematical definition.} We propose a formal definition of \emph{optimistic value functions} in MARL, grounded in convex risk measures. We show how divergence-penalized dual forms naturally yield optimism as controlled risk-seeking. 
%   \item \textbf{Policy-gradient theorem.} We derive a policy-gradient theorem for optimistic value functions, including explicit sample-based formulas in the entropic risk/KL penalty case.
%   \item \textbf{Algorithms.} We design decentralized optimistic actor-critic algorithms informed by the above theory, and validate them empirically on cooperative benchmarks---showing improved coordination compared to both risk-neutral baselines and heuristic optimistic methods.
% \end{itemize}

\section{Preliminaries}
This section sets up notation for MDPs and their multi-agent extensions, introduces the optimistic value functions that are central to our work, and recalls the basic properties of convex risk measures.
\subsection{Markov Decision Processes (MDPs)} 
A finite Markov decision process (MDP) is defined as a tuple~$\cM = (\cS, \cA, P, r, \gamma, \rho)$, where~$\cS$ denotes a finite set of states,~$\cA$ a finite set of actions,~$P(s'|s,a)$ the transition probability of moving from state~$s$ to~$s'$ when action~$a$ is taken,{~$r: \cS \times \cA \to \bR$ the reward function,~$\gamma \in [0,1)$ the discount factor, and~$\rho$ the initial state distribution over~$\cS$.} A stochastic policy $\pi: \cS\to \Delta^{|\cA|}$  specifies a strategy where the agent chooses its action based on the current state in a stochastic fashion; specifically, the probability of choosing action~$a$ at state~$s$ is given by~$\Pr(a|s) = \pi(a|s)$. 
For notational simplicity, we use~$\pi_s(\cdot)$ as shorthand for~$\pi(\cdot | s)$. 
For a given stationary policy~$\pi$ , we denote the discounted state visitation distribution by 
\begin{align*}
  \textstyle  d^{\pi}(s):=(1-\gamma)\sum_{t=0}^{+\infty}\gamma^t \Pr(s_t = s ~|~ s_0\sim\rho, a_\tau\sim \pi_{s_\tau}).
\end{align*}

\subsection{Multi-agent MDPs} 
In multi-agent settings the action space factors as~$\cA = \cA_1\times\cA_2\times\cdots\times \cA_n$, so a joint action~$a=(a_1, a_2,\cdots,a_n)$ is composed of individual agent actions~$a_i$.
We focus on decentralized product policies
\begin{align}\label{eq:decentralized-policy}
   \textstyle  \pi^\theta(a|s) = \prod_{i=1}^n \pi_i^{\theta}(a_i|s),
\end{align}
where parameters are grouped as~$\theta = (\theta_1, \theta_2,\cdots,\theta_n)$.
A \emph{direct parametrization} means each~$\theta_i\in \bR^{\cS\times\cA_i}$ encodes action probabilities explicitly, i.e.,
\begin{align}\label{eq:direct-parameterization}
    \textstyle \theta_{s, a_i} = \pi_i(a_i|s),
\end{align}
where, for notational simplicity, we abbreviate~$\theta_i(s,a_i)$ as~$\theta_{s,a_i}$.
\subsection{Optimistic Value Functions} 
To encourage cooperation in decentralized learning, we evaluate a baseline policy~$\pi$ via \emph{optimistic value functions} and~$Q$-functions that allow deviations penalized by a divergence.
For auxiliary policies~$\{\hat{\pi}_t\}_{t\geq 0}$,
\begin{small}
\begin{align}
        &\textstyle V^\pi\!(s) \!= \!\max_{\hpi}\! \left[\!\bE_{s_t\!, a_t\!\sim\! \hpi}\!\sum_{t\!=\!0}^{+\!\infty}\!\gamma^t\! \left(r(s_t, \!a_t) \!-\!\! D(\hpi_{t, s_t}||\pi_{s_t})\right)\!\Big|s_0\!=\!s\right],\notag\\
        &\textstyle Q^\pi(s,a) =r(s,a)+ \max_{\{\hpi_t\}_{t\ge1}} \notag \\
        &\textstyle \left[\bE_{s_t, a_t\sim \hpi}\!\sum_{t=1}^{+\infty}\!\gamma^t\! \left(r(s_t, a_t) \!-\!\! D(\hpi_{t, s_t}||\pi_{s_t})\right)\!\Big|s_0\!=\!s, a_0\!=\!a\right].\!\!\!\label{eq:def-optimistic-value-function}\!\!
\end{align}
\end{small}
Here,~$D(\hat{\pi} || \pi)$ quantifies the deviation of the auxiliary policy~$\hat{\pi}$ from the baseline policy~$\pi$, for instance through the KL divergence.
The term \emph{optimistic} reflects the maximization over the auxiliary policy sequence~${\hat \pi_t}$ in the definitions above.
Rather than directly evaluating the return of~$\pi$,the value functions~$V^\pi$ and~$Q^\pi$ incorporate deviations~$\hat{\pi}$ that maximize the expected cumulative reward, while paying a penalty proportional to how far~$\hat{\pi}$ strays from~$\pi$.
This construction encourages the agent to envision the best-case performance of~$\pi$ under locally improved decisions, yielding an evaluation that is deliberately biased toward potentially higher-performing behaviors.
The advantage function is defined as~$A^\pi(s,a) = Q^\pi(s,a) - V^\pi(s)$.

\subsection{Convex Risk Measures}
Let~$\cF$ be the set of real-valued functions on a finite action set~$\cA$.
A map~$\sigma:\cF \to \bR$ is a convex risk measure if it satisfies:

%Given a probability distribution $\mu\in \Delta^{|\cS|}$ on a finite space $\cS$, the convex measure of risk  $\sigmu: \bR^{|\cS|}\to \bR$ is a function over the set of real-valued functions defined over $\cS$ (without loss of generality, we use $\bR^{|\cS|}$ to denote the set of all real-valued functions over $\cS$) that satisfies the following properties:%\Lina{this definition is unclear. What is $\cF$? not introduced yet. For this intro, not only mathematically, but also a brief sentence of the intuitive meaning. Also it says that the risk measure depends on $\mu$ but the three properties doesn't show that. how $\mu$ does show up in the definition of convex risk measure? Also need to briefly explain $\bR^{|\cS|}$ to make the job easier for readers }

\begin{enumerate}
\item \emph{Monotonicity}: For~$f', f\in \cF$, if~$f'\le f$, then~$\sigma(f)\le \sigma(f').$
\item \emph{Translation invariance}: for any $f\in \cF, m \in \bR$, $\sigma(f + m) = \sigma(f) - m$ .
\item \emph{Convexity}: for any $f', f\in \cF, \lambda\in[0,1]$, $\sigma(\lambda f + (1-\lambda) f') \le \lambda\sigma(f) + (1-\lambda) \sigma(f').$
\end{enumerate} 
By standard duality theory, convex risk measures admit the following representation~\cite{follmer2002convex}.
\begin{theorem}[Dual Representation Theorem \cite{follmer2002convex}]\label{thm:convex-risk-dual-representation}
The function $\sigma:\cF\to \bR$ is a convex risk measure if and only if there exists a ``penalty function'' $D(\cdot): \Delta^{|\cA|}\to \bR$ such that 
\begin{equation}\label{eq:dual-representation-d-to-V}
   \textstyle \sigma(f) = \sup_{\widehat\pi\in\Delta^{|\cA|}} \left(-\bE_{\widehat\pi} f - D(\widehat\pi)\right).
\end{equation}
In specific, $D$ can be written in the following form:
\vspace{-5pt}
\begin{align}\label{eq:dual-representation-V-to-d}
   \textstyle  D(\widehat\pi) =  \sup_{f\in\cF}\left( -\sigma(f) - \bE_{a\sim \widehat\pi}f(a)\right).
\end{align}
\end{theorem}
{Note that $\sigma$ and $D$ serve as the Fenchel conjugate of each other.} In most cases, a convex risk measure $\sigma(f)$ can be understood as quantifying the risk of a random variable taking values $f(a)$, where $a$ is sampled from some distribution $a \sim \pi$. Thus, commonly used risk measures are naturally tied to an underlying probability distribution $\pi \in \Delta^{|\cA|}$ (see, e.g., Example \ref{example:entropy-risk-measure}). In this work, we focus on such distribution-dependent risk measures, and therefore write $\sigma(\pi,\cdot)$ to emphasize the dependence on $\pi$. Correspondingly, in the dual representation theorem we denote the penalty term $D(\hat{\pi})$ of $\sigma(\pi,\cdot)$ by $D(\hat{\pi}||\pi)$. \footnote{The notation $D$ is intentionally overloaded: it denotes both the regularization term in \eqref{eq:def-optimistic-value-function} and the penalty function for a risk measure in \eqref{eq:dual-representation-d-to-V}. The connection between these two uses will become clear in subsequent sections.}

Here we provide an example of convex risk measure and its dual form. %\Lina{I went to \cite{follmer2002convex}, but I couldn't find exactly the same theorem}
\begin{example}[Entropy risk measure \cite{follmer2002convex}]\label{example:entropy-risk-measure}
For a given $\beta>0$, the entropy risk measure takes the form: 
\begin{equation*}
\textstyle \sigma(\pi, f) =  \beta^{-1}\log\bE_{a\sim \pi}e^{-\beta f(a)}.
\end{equation*}
Its corresponding penalty function $D$ in the dual representation theorem is the KL divergence
\begin{equation*}
  \textstyle  D(\widehat\pi||\pi) = \beta^{-1}\KL(\widehat\pi||\pi) = \beta^{-1}\sum_{a\in\cA} \widehat\pi(a)\log\left({\widehat\pi(a)}/{\pi(a)}\right).
\end{equation*}
\end{example}

\section{Bellman Equation and Policy Gradient Theorem for Optimistic Value Functions}
This section develops the theoretical foundation for optimizing the optimistic value functions introduced earlier in \cref{eq:def-optimistic-value-function}. 
We first establish a Bellman equation that characterizes these functions, and then derive a policy-gradient theorem that enables their optimization.
\begin{lemma}[Bellman Equation]\label{lemma:Bellman} 
The optimistic value functions and~$Q$-functions in \cref{eq:def-optimistic-value-function} satisfy the following Bellman recursion:
 \begin{equation}
\begin{split}
      \textstyle V^\pi(s)  &= \max_{\hpi}\sum_a\hpi(a|s)Q^\pi(s,a) - D(\hpi_s||\pi_s)\\
      \textstyle\qquad ~~ &=\sigma(\pi_s, -Q^\pi(s,\cdot)),\\
    \textstyle Q^\pi(s,a) &= r(s,a) + \gamma\sum_{s'}P(s'|s,a) V^\pi(s') ,
\end{split}\label{eq:optimistic-value-function-bellman}
\end{equation}
where~$\sigma(\pi_s,\cdot)$ is the convex risk measure associated with the penalty~$D$, via \cref{eq:dual-representation-d-to-V}. 
Moreover, the auxiliary optimistic distribution~$\hpi_t$ in \cref{eq:def-optimistic-value-function} can be taken as a stationary policy:
    \begin{align}
      \textstyle  \hpi(\cdot|s) = \argmax_{\hpi}\sum_a\hpi(a|s)Q^\pi(s,a)- D(\hpi_{s}||\pi_{s}).
    \end{align}\label{eq:def-hpi}
\end{lemma}
\begin{proof}
\vspace{-10pt}
 From the definition of~$V^\pi$ in \cref{eq:def-optimistic-value-function} and dual representation theorem we get
    \begin{align*}
       & \textstyle V^\pi\!(s) \!= \!\max_{\hpi_0} \bE_{a_0\sim\hpi_0} \Bigg[r(s_0,a_0) - D(\hpi_{0,s_0}||\pi_{s_0}) + \\
       &\textstyle\left. \max_{\{\hpi_t\}_{t\ge1}}\left[\!\bE_{s_t\!, a_t\!\sim\! \hpi}\!\sum_{t\!=\!1}^{+\!\infty}\!\gamma^t\! \left(r(s_t, \!a_t) \!-\!\! D(\hpi_{t, s_t}||\pi_{s_t})\right)\!\Big|s_0\!=\!s\!\right]\right]\\
        &\textstyle= \!\max_{\hpi_0} \bE_{a_0\sim\hpi_0} \left[Q^\pi(s_0,a_0) - D(\hpi_{0,s_0}||\pi_{s_0})\Big|s_0=s\right]\\
        &\textstyle=\max_{\hpi}\sum_a\hpi(a|s)Q^\pi(s,a) - D(\hpi_s||\pi_s) \\
      &\textstyle=\sigma(\pi_s, -Q^\pi(s,\cdot)).
    \end{align*}
    Further we know that the optimistic $\hpi$ is given by
    \begin{align*}
      \textstyle  \hpi(\cdot|s) = \argmax_{\hpi}\sum_a\hpi(a|s)Q^\pi(s,a) - D(\hpi_s||\pi_s).
    \end{align*}
    Similarly, from the definition of~$Q^\pi$ in \eqref{eq:def-optimistic-value-function} we get that
\begin{align*}
    &Q^\pi(s,a) = r(s,a) + \\
    &\max_{\{\hpi_t\!\}_{t\!\ge\!1}}\! \!\left[\!\bE_{s_t, a_t\sim \hpi}\!\sum_{t=1}^{+\infty}\!\!\gamma^t\! \left(r(s_t, a_t) \!-\!\! D(\hpi_{t, s_t}||\pi_{s_t})\right)\!\Big|s_0\!=\!s, \!a_0\!=\!a\!\right]\\
    &= r(s,a) + \\
    & \bE_{s'\!\sim \!P(\!\cdot|s,a)}\!\max_{\{\!\hpi_t\!\}_{t\!\ge\!1}}\! \!\left[\!\bE_{s_t, \!a_t\!\sim\! \hpi}\!\sum_{t=1}^{+\infty}\!\!\gamma^t\! \left(r(\!s_t, a_t) \!-\!\! D(\hpi_{t, s_t}||\pi_{s_t})\right)\!\Big|s_1\!=\!s'\!\right]\\
   &= r(s,a) + \gamma\sum_{s'}P(s'|s,a) V^\pi(s'),
\end{align*}
which completes the proof.
\end{proof}

We are now ready to state the policy gradient theorem for the optimistic value function as follows:
\begin{theorem}[Policy gradient theorem for optimistic value function]\label{theorem:policy-gradient-optimistic-value-function}
For a parametrized policy~$\pi^\theta$:
    \begin{align*}
      &\textstyle   \nabla_\theta V^{\pi^\theta}(s_0) = \\
      &\textstyle\frac{1}{1-\gamma}\sum_{s} d^{\hpi^\theta}_{s_0}(s) \sum_{a} \frac{\partial \sigma(\pi_s^\theta, -Q^{\pi^\theta}(s,\cdot))}{\partial \pi^\theta(a|s)}\pi^\theta(a|s)\nabla_\theta\log\pi^\theta(a|s),
    \end{align*}
    where~$ d^{\hpi^\theta}_{s_0}$ is the discounted visitation distribution under the optimistic policy~$\hat{\pi}^\theta$ defined in \eqref{eq:def-hpi}.
    In the special case of the entropic risk measure
    \begin{equation*}
        \sigma(\pi_s, -Q(s,\cdot)) =  \beta^{-1}\log\bE_{a\sim \pi_s}e^{\beta Q(s,a)},
    \end{equation*}
    the policy gradient simplifies to
    \begin{align*}
        &\textstyle \nabla_\theta V^{\pi^\theta}(s_0)=\\&\qquad
\textstyle        \frac{1}{\beta(1-\gamma)}\sum_{s,a} d^{\hpi^\theta}_{s_0}(s) e^{\beta A^{\pi^\theta}(s,a)}\pi^\theta(a|s)\nabla_\theta\log\pi^\theta(a|s),
    \end{align*}
    with~$\hpi^\theta(a|s) \propto \pi^\theta(a|s) e^{\beta Q^{\pi^\theta}(s,a)}$.
\end{theorem}
The proof of \cref{theorem:policy-gradient-optimistic-value-function} (deferred to the Appendix) is done by taking the gradient of $\theta$ on both sides of the Bellman equation \eqref{eq:optimistic-value-function-bellman}. 
\begin{rmk}[Interpretation]
The Bellman equation above mirrors the risk-neutral case but replaces the baseline expectation under~$\pi$ with a risk-sensitive evaluation under an auxiliary distribution~$\hat{\pi}$.
This makes the value of~$\pi$ \emph{optimistic}-tilted toward actions or trajectories that could yield higher returns, while still penalizing large deviations from~$\pi$.
The policy-gradient theorem retains the same structural form as the classical result, but introduces three important differences. 
First, the visitation distribution is taken under~$\hat{\pi}$ rather than~$\pi$, emphasizing more favorable trajectories.
Second, action weights are exponentially scaled by~$e^{\beta A^\pi(s,a)}$, amplifying high-advantage actions and inducing a principled risk-seeking bias.
Third, as~$\beta\to 0$, we recover the standard policy-gradient theorem:~$\hat{\pi}\to \pi$ and the exponential factor reduces to the advantage~$A^\pi(s,a)$.
Together, these results establish a rigorous connection between optimism and risk-seeking convex risk measures, providing the foundation for optimistic policy-gradient algorithms.
\end{rmk}
\section{Optimistic update for multi-agent learning}
In this section, we leverage the policy gradient theorem (\cref{theorem:policy-gradient-optimistic-value-function}) to design optimistic RL algorithms for decentralized multi-agent settings.
Throughout this section, we assume direct parametrization \eqref{eq:direct-parameterization} of the agents' policies, and we focus on the case where the penalty function is the KL divergence, i.e.,~$D(\hpi||\pi) = \beta^{-1}\KL(\hpi||\pi)$, so that the associated risk measure~$\sigma$ is the entropic risk.
Before presenting the update rules, we introduce two key quantities.
For each agent~$i$, define the \emph{averaged optimistic~$Q$-function} and the \emph{averaged optimistic advantage function} as
\begin{align}
\textstyle \oQ_i^{\pi^\theta}(s,a_i) = \beta^{-1}\bE_{a_{-i}\sim \pi_{-i}^\theta(a_{-i}|s)} e^{\beta Q^{\pi^\theta}(s,a_i,a_{-i})}\label{eq:def-oQ},\\
      \vspace{-10pt}
  \textstyle \oA_i^{\pi^\theta}(s,a_i) = \beta^{-1}\bE_{a_{-i}\sim \pi_{-i}^\theta(a_{-i}|s)} e^{\beta A^{\pi^\theta}(s,a_i,a_{-i})}\label{eq:def-oA}
    \end{align}
%We are now ready to state the policy gradient theory for multi-agent MDP with decentralized policy.
\begin{lemma}[Multi-agent optimistic policy gradient]
\label{lemma:multi-agent-PG}
Consider decentralized policy \eqref{eq:decentralized-policy} with direct parameterization \eqref{eq:direct-parameterization}. 
We have that
\begin{align}\label{eq:gradient-decentralized-direct}
\textstyle  \frac{\partial V^{\pi^\theta}(s)}{\partial \theta_{s,a_i}} =\frac{1}{1-\gamma} d^{\hpi^\theta}(s) \oA_i^{\pi^\theta}(s,a_i),
    \end{align}
\end{lemma}
\begin{proof}
    This is a direct corollary of the general \cref{theorem:policy-gradient-optimistic-value-function}, obtained by substituting~$\pi^\theta$ to decentralized policy \eqref{eq:decentralized-policy} with direct parameterization \eqref{eq:direct-parameterization}.
\end{proof}

\begin{rmk}[Interpretation]
The quantities~$\oQ_i^{\pi^\theta}$~$\oA_i^{\pi^\theta}$ summarize agent~$i$'s \emph{optimistic evaluation} of its actions, averaged over the behavior of the other agents under the baseline policy~$\pi_{-i}^\theta$.
Intuitively,~$\oQ_i^{\pi^\theta}(s,a_i)$ measures the risk-seeking value of choosing~$a_i$, while~$\oA_i^{\pi^\theta}(s,a_i)$ captures the deviation of that choice from the state value.
Both tilt the evaluation toward joint actions that promise higher returns, thus embodying the cooperative bias induced by optimism\footnote{  Note that in this paper we consider the identical reward setting, and thus in this formulation all agents evaluate actions under the same reward landscape, their optimistic tendencies become naturally aligned, promoting coordinated exploration rather than miscoordination. Moreover, because our policy updates are incremental, optimism gradually guides agents toward consistent high-reward policies, further reducing the likelihood of misaligned behavior.}.

{Interpreting the optimistic auxiliary joint policy $\hpi$ in the multi-agent setting is more nuanced than in the single-agent case. Ideally, in the multi-agent setting, one would like to express $\hpi$ as a factorized product of per-agent optimistic policies, i.e., $\hpi(\cdot|s) = \prod_{i=1}^n\hpi_i(\cdot|s)$, and the corresponding optimistic value function would be defined as
\begin{small}
    \begin{align*}
    V^{\pi}\!(s) \!=\!\!\!\!\!\! \max_{\hpi \!=\! (\!\hpi_{1
}\!, \dots,\! \hpi_n\!)} \!\!\left[\!\bE_{s_t, a_t\sim \hpi}\!\!\sum_{t=0}^{+\infty}\!\!\gamma^t\!\! \left(\!r(s_t, a_t) \!-\!\! \sum_i\!\!D(\hpi_{i, s_t}||\pi_{i,s_t}\!)\!\!\right)\!\!\Big|s_0\!=\!s\!\right]\!\!.\notag
\end{align*}
\end{small}
However, this decentralized formulation is generally intractable to optimize. To retain computational feasibility, we relax the constraint that $\hat{\pi}$ must factorize across agents, allowing it to represent a joint auxiliary policy. This relaxation preserves scalability and still captures the cooperative optimism effect we aim to model.} %We agree that developing a more explicit and interpretable formulation of optimism in the multi-agent setting, particularly one that preserves both tractability and theoretical alignment, is an important and interesting direction for future research. 
\end{rmk}

\begin{rmk}[Practical approximation]
In practice, computing~$d^{\hat\pi}$ and~$\oA_i^\pi$ exactly is intractable.
A common surrogate is to replace~$d^{\hpi}(s)\oA_i^\pi(s,a)$ with~$\oQ_i^\pi(s,a)$ in the update, yielding a gradient-like rule that is easier to estimate from data.
Specifically, we consider the following updates for agent~$i$.
First, the \emph{policy-gradient style update}:
\begin{align}\label{eq:optimistic-Q-PG}
\textstyle \theta_{i,s}^{(t+1)} = \Proj_{\Delta^{\cA_i}} \Big(\theta_{i,s}^{(t)} + \frac{\eta}{1-\gamma}\oQ_i^{\pi^\theta}(s,\cdot)\Big),
\end{align}
Second, the \emph{policy-iteration/Q-learning style update}:
\begin{align}\label{eq:optimistic-Q-learning}
\textstyle \pi_i(s) = \arg\max_{a_i} \oQ_i^{\pi^\theta}(s,a_i).
\end{align}
Both formulations are consistent with the optimistic policy-gradient principle and share the same stationary points (since $d^{\hat\pi}(s)\oA_i^\pi(s,a) = \tfrac{d^{\hat\pi}(s)}{e^{\beta V^\pi(s)}} \oQ_i^\pi(s,a)$).
Thus, the resulting updates can be interpreted as a form of scaled gradient ascent, where the exact policy-gradient direction is preserved up to a positive multiplicative constant, making them more practical to implement in decentralized settings.
\end{rmk}

% Directly computing $d^{\hat\pi,P}$ and $\oA$ is challenging in practice, so we adopt a tractable approximation that replaces $\oA$ with $\oQ$ in a gradient-like update. Specifically, we update the policy parameters via
% \begin{align}\label{eq:optimistic-Q-PG}
% \theta_{i,s}^{(t+1)} = \Proj_{\Delta^{\cA_i}} \Big(\theta_{i,s}^{(t)} + \frac{\eta}{1-\gamma}\oQ_i^{\pi^\theta}(s,\cdot)\Big),
% \end{align}
% or equivalently perform a policy-iteration or Q-learning style update
% \begin{align}\label{eq:optimistic-Q-learning}
% \pi_i(s) = \arg\max_{a_i} \oQ_i^{\pi^\theta}(s,a_i).
% \end{align}
% Both updates are consistent with the original optimistic policy gradient and converge to the same stationary points, while being easier to implement and estimate in a decentralized multi-agent setting.
% Given that $d^{\hpi,P}$ is relatively hard to estimate, we first focus on the following modified gradient update
% \begin{align}
%     \theta_{i,s}^{(t+1)} = \Proj_{\Delta^{\cA_i}} \left(\theta_{i,s}^{(t)} + \frac{\eta}{1-\gamma}\oQ_i^{\pi^\theta}(s,\cdot)\right)  \label{eq:optimistic-gradient-update}
% \end{align}
% TODO: Interpret the averaged advantage function. Inspired by policy gradient, we also define the averaged Q function. Note that it is the same performing gradient update with respect to Q function comparing rather than the policy gradient
\vspace{-10pt}
\subsection{Decentralized Learning Algorithm}
Building on \cref{lemma:multi-agent-PG}, we now design a sample-based decentralized optimistic RL algorithm.
From the update rules in~\eqref{eq:optimistic-Q-PG}-\eqref{eq:optimistic-Q-learning}, the main challenge lies in accurately estimating the averaged optimistic~$Q$-functions~$\oQ_i^\pi(s,a_i)$.
The following lemma introduces an auxiliary variable that enables a tractable Bellman-style recursion.

\begin{lemma}[Decentralized Bellman equation]
\label{lemma:decentralized-learning}
Define an auxiliary variable~$Z^{\pi}(s):= e^{\beta V^{\pi}(s)}$.
If transitions are deterministic, i.e.,~$P(s'|s,a) = \mathbf{1}\{s' = f(s,a)\}$, then the averaged optimistic quantities satisfy:
    \begin{align}
        \oQ_i^{\pi}(s,a_i) &= \beta^{-1}\bE_{a_{-i}\sim\pi_{-i}} e^{\beta r(s,a_i,a_{-i})} \left(Z^{\pi}(f(s,a_i,a_{-i}))\right)^{\gamma}\notag,\\
        Z^{\pi}(s) &= \bE_{a_i\sim\pi_i^\theta} \oQ_i^{\pi}(s,a_i).\label{eq:decentralized-Bellman}
    \end{align}
\end{lemma}
\begin{proof}
    From the definition of~$\oQ_i^{\pi}(s,a_i)$ we have:
    \begin{align*}
        \oQ_i^{\pi}(s,a_i) 
        %&= \beta^{-1}\bE_{a_{-i}\sim \pi_{-i}(a_{-i}|s)} e^{\beta Q^{\pi}(s,a_i,a_{-i})}\\
        &= \beta^{-1}\bE_{a_{-i}\sim \pi_{-i}(a_{-i}|s)} e^{\beta \left(r(s,a)+\gamma V^{\pi}(f(s,a))\right)}\\
        &=\beta^{-1}\bE_{a_{-i}\sim \pi_{-i}(a_{-i}|s)} e^{\beta \left(r(s,a)+\gamma\beta^{-1}\log(Z^{\pi}(f(s,a)))\right)}\\
        &=\beta^{-1}\bE_{a_{-i}\sim\pi_{-i}} e^{\beta r(s,a_i,a_{-i})} \left(Z^{\pi}(f(s,a_i,a_{-i}))\right)^{\gamma}.
    \end{align*}
    Further, from the definition of~$Z^{\pi}(s)$ we have:
    \begin{align*}
        Z^{\pi}(s) &= e^{\beta V^{\pi}(s)} = e^{\beta \beta^{-1}\log(\bE_{a\sim\pi} e^{\beta Q(s,a)})}\\
       % &= \bE_{a_i\sim\pi^\theta} \bE_{a_{-i}\sim \pi_{-i}} e^{\beta Q(s,a_i,a_{-i})}\\
        &= \bE_{a_i\sim\pi^\theta} \oQ_i^{\pi}(s,a_i),
    \end{align*}
    which completes the proof.
\end{proof}

\begin{rmk}[Interpretation]
\cref{lemma:decentralized-learning} shows that the auxiliary variable~$Z^\pi(s)$ acts as an exponential proxy for the state value.
This avoids computing expectations over the full joint space directly, since each agent can evaluate~$\oQ_i^\pi(s,a_i)$ by averaging only over the actions of the other agents.
In practice, the expectations are replaced with empirical averages from sampled trajectories, leading to a practical decentralized policy-evaluation procedure.
This is summarized in \cref{alg:optimistic-q-multi}, that iteratively updates both $\oQ_i^\pi(s,a_i)$ and $Z^\pi(s)$.
Once we have reliable estimates of~$\oQ_i^\pi(s,a_i)$ from this evaluation step, we can improve the agent’s policy using either policy gradient or policy iteration. 
\cref{alg:optimistic-RL} formalizes this process.%, where the computed~$\oQ_i^\pi(s,a_i)$ serves as a guide for updating each agent’s policy in a decentralized manner, enabling optimistic, coordinated learning across agents.
\end{rmk}

 % algorithmicx style

\begin{algorithm}[h]
\caption{\\Agent $i$'s Decentralized Optimistic Policy Evaluation}
\label{alg:optimistic-q-multi}
\begin{algorithmic}[1]
\Require Discount $\gamma\in(0,1)$, stepsize sequence $\{\alpha_t\}$, optimistic parameter $\beta>0$, horizon $T_Q$, and agent $i$'s policy $\pi_i$
\State Initialize $Z^{\pi}(s), \oQ_i^{\pi}(s,a)$ for all $s,a$, \Statex \qquad {e.g., $Z(s)\leftarrow 1, \oQ_i^{\pi}(s,a)\leftarrow0$}
    \State Sample initial state $s_0{\sim \rho}$
    \For{$t=0$ to $T_Q-1$} 
        \State Each agent $i$ sample action $a_{i} \sim \pi_i(\cdot\mid s_t)$, then execute $a_{i,t}$, observe reward $r_t$ and next state $s_{t+1}$ 
            \State \textbf{Q-update:}
            \[
              \textstyle  \oQ_i^{\pi}(s_t,a_{i,t})
                \gets (1-\alpha_t)\,\oQ_i^{\pi}(s_t,a_{i,t})
                + \!\frac{\alpha_t}{\beta}\, e^{\beta r_t}\,\Big(\!Z^{\pi}(s_{t+1})\!\Big)^{\!\gamma}
            \]
            \State \textbf{Z-update:}
            \[
                Z^{\pi}(s_t)
                \gets (1-\alpha_t)\, Z^{\pi}(s_t)
                + \alpha_t\, \oQ_i^{\pi}(s_t,a_{i,t})
            \]
\EndFor
\end{algorithmic}
\end{algorithm}

\begin{algorithm}[h]
\caption{\\Agent $i$'s Decentralized Optimistic Policy Update}
\label{alg:optimistic-RL}
\begin{algorithmic}[1]
\Require Initial Policy $\pi = \prod_{i=1}^n \pi_i$
    \For{$t=0$ to $T-1$ and each agent $i$} 
        \State Estimate $\oQ_i^\pi(s,a_i)$ using optimistic policy evaluation (Algorithm \ref{alg:optimistic-q-multi})
            \State \textbf{Policy Gradient:}       
\begin{talign*}
    \theta_{i,s}^{(t+1)} &= \Proj_{\Delta^{\cA_i}} \left(\theta_{i,s}^{(t)} + \frac{\eta}{1-\gamma}\oQ_i^{\pi}(s,\cdot)\right), ~~~\pi\leftarrow \pi^{\theta^{(t+1)}}
\end{talign*}
~~or \textbf{Policy Iteration / Q-learning:}
\begin{talign*}
    \pi_i(s) &\leftarrow \argmax_{a_i} \oQ_i^{\pi}(s,a_i)
\end{talign*}
\EndFor
\end{algorithmic}
\end{algorithm}
\vspace{-20pt}
\section{Numerical Result} 
\vspace{-5pt}
We now provide empirical evidence for the benefits of optimistic updates.
Our experiments include a simple gridworld testbed, which illustrates the effect of optimism in escaping suboptimal equilibria, and a cooperative continuous-control benchmark, which evaluates performance in a more realistic multi-agent setting.

\subsection{Full-Information Setting}
We first evaluate the modified optimistic gradient update \eqref{eq:optimistic-Q-PG} under the full-information setting, where $\oQ_i^\pi$ is computed exactly, on a simple two-agent gridworld.
Each agent chooses actions~$a_x,a_y\in \{-1,0,1\}$, and the joint state~$s=(x,y)$ evolves as~$x' = \mod(x + a_x, 4)$,~$y' = \mod(y + a_y, 4)$.
At each position, agents receive reward~$R(x,y)$ as shown in \cref{tab:reward-numerics} (left).
The global optimal policy is for agents to transit to~{$(2,2)$} as fast as possible.
However, staying at location~$(4,4)$ represents a suboptimal Nash equilibrium. { Note that for the optimal location $(2,2)$, when the other agent’s action is uncertain or noisy—e.g., it may deviate with some probability to location 1 or 3—they incur a penalty of $-10$. In contrast, remaining at the suboptimal Nash equilibrium $(4,4)$ is less sensitive to deviations by the other agent. Consequently, risk-neutral algorithms often concentrate on this suboptimal equilibrium, a phenomenon known as relative overgeneralization (RO). The optimistic update mitigates this issue by biasing agents toward higher-reward outcomes, enabling better convergence to the globally optimal equilibrium, even when it involves higher cooperative risk. This is validated in \cref{fig:neutral-alg} and \cref{fig:optimistic-alg}, which show the normalized state-visitation distributions~$d^{\pi}$ (where~$\pi$ is the policy found by the algorithm) under the learned policies. \footnote{We would like to clarify that Figure 2 is \emph{not} a special case from a particular trial, but rather depicts the heatmap of the \emph{ground-truth} state visitation distribution $d^{\pi}$. Hence, it represents the overall statistical distribution induced by the learned policy, not the outcome of a single run.}
 The risk-neutral algorithm tends to concentrate mass on the suboptimal equilibrium~$(4,4)$.
By contrast, the optimistic update assigns most probability to the true optimum~{$(2,2)$}.}

% \begin{table}[h]
%   \centering
%   \begin{equation*}
%     \begin{bmatrix}
%       -10 & -10 & -10 & 0 \\
%       -10 & 10 & -10& 0 & \\
%       -10 & -10 & 0 & 0 \\
%        0& 0 & 0& 5\\
%     \end{bmatrix}
%   \end{equation*}
%   \caption{Reward table $R$}
%   \label{tab:reward-numerics}
% \end{table}

% value: 181.87, 120.66

\begin{figure}[tb]
  \centering
  % Left: reward table
  \begin{minipage}{0.35\linewidth} % squeeze width
\begin{small}
    \centering
    $\displaystyle
    \begin{bmatrix}
      -10 & -10 & -10 & 0 \\
      -10 & 10 & -10 & 0 \\
      -10 & -10 & 0 & 0 \\
      0 & 0 & 0 & 5
    \end{bmatrix}$
    \caption{Reward table $R$}
    \label{tab:reward-numerics}% no separate number
 \end{small}
  \end{minipage}%
  \hfill
  % Middle: risk-neutral heatmap
  \begin{minipage}{0.3\linewidth}
    \centering
    \includegraphics[width=\linewidth]{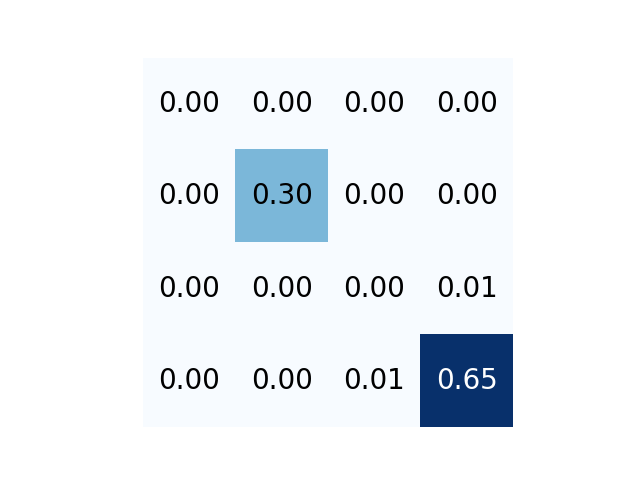}
    \caption{Risk-neutral}
    \label{fig:neutral-alg}
  \end{minipage}%
  \hfill
  % Right: optimistic heatmap
  \begin{minipage}{0.3\linewidth}
    \centering
    \includegraphics[width=\linewidth]{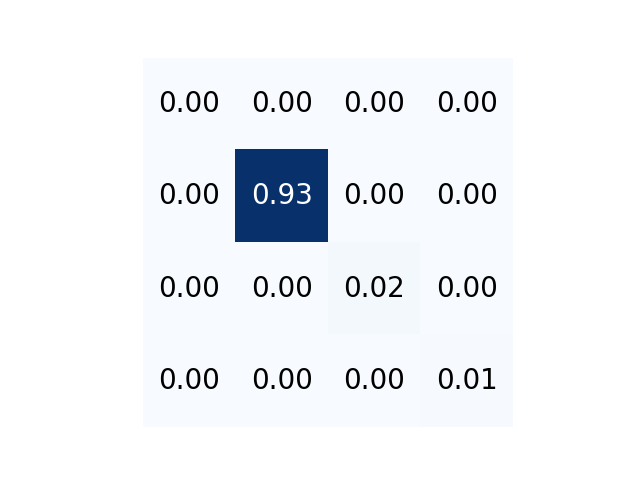}
    \caption{Optimistic}
    \label{fig:optimistic-alg}
  \end{minipage}
\end{figure}

%We run the optimistic modified gradient update \eqref{eq:optimistic-Q-PG} and the corresponding risk-neutral update and compare the performance. 
%Figure \ref{fig:neutral-alg} and \ref{fig:optimistic-alg} plots the heatmap of the normalized state visitation distribution $d^{\pi,P}$ (where $\pi$ is the policy found by the algorithm). The figure indicates that the risk-neutral algorithm tends to assign more probabilty to the suboptimal state $x=4,y=4$ (with an total reward of 120.66), whereas the optimistic algorithm assigns most probability to the optimal state $x=2,y=2$ (with a total reward of 181.87, which is the same as the total reward of global policy.)
% \begin{figure}[htbp]
%   \begin{minipage}{0.45\linewidth}
%     \centering
%     \includegraphics[width=\linewidth]{figures/neutral-alg-heatmap.png}
%     \caption{Risk-neutral}
%     \label{fig:neutral-alg}
%   \end{minipage}%
%   \hfill
%   \begin{minipage}{0.45\linewidth}
%     \centering
%     \includegraphics[width=\linewidth]{figures/optimistic-alg-heatmap.png}
%     \caption{Optimistic algorithm}
%     \label{fig:optimistic-alg}
%   \end{minipage}
% \end{figure}
\vspace{-5pt}
\subsection{Sample-based Setting}
\begin{figure}[htbp]
    \centering
    \includegraphics[width=0.7\linewidth]{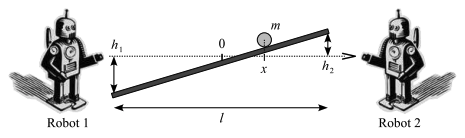}
    \centering
    \caption{\centering Cooperative Ball Balancing \cite{matignon2007hysteretic}}
    \label{fig:ball-balancing}
\end{figure}
We next evaluate the decentralized sample-based algorithm on the cooperative ball-balancing task (\cref{fig:ball-balancing}), where two agents must keep a ball centered on a flat table they jointly control.
This benchmark is known to require coordination and has been widely used to test cooperative MARL methods~\cite{matignon2007hysteretic,Taniguchi05Adaptive}.
\cref{tab:comparison-results} compares final returns across 10 random seeds.
The optimistic algorithms consistently outperform both decentralized~$Q$-learning and hysteretic~$Q$-learning.
The best-performing setting ($\beta=0.04$) achieves the highest mean return (169.2) with the lowest variance.
Performance is robust to moderate changes in~$\beta$, showing that optimism need not be finely tuned to be effective.
Hysteretic~$Q$-learning performs competitively but slightly worse, while decentralized~$Q$-learning exhibits higher variance and lower mean performance.

These experiments highlight the advantage of optimistic updates: they not only steer agents away from suboptimal equilibria (gridworld), but also improve both \emph{efficiency} and \emph{stability} in challenging cooperative tasks (ball balancing).
%We also tested the sampled-based version of the algorithm on a coorperative ball balance problem (Fig \ref{fig:ball-balancing}) whose purpose is to keep the balance of a rolling ball in the center of a
%flat table holding by two robots at the extremities \cite{matignon2007hysteretic,Taniguchi05Adaptive} and compare it with classical methods such as decentralized Q-learning and hysteretic Q-learning. To make the different algorithm more comparable, we use policy iteration update in Algorithm \ref{alg:optimistic-q-multi}. The comparison of different methods are shown in Table \ref{tab:comparison-results}.

\begin{table}[tb]
\centering
\vspace{10pt}
\caption{Final performance comparison (Mean $\pm$ Std).}
\label{tab:comparison-results}
\begin{tabular}{lc}
\toprule
\textbf{Method} & \textbf{Mean $\pm$ Std} \\
\midrule
\textbf{Optimistic $\mathbf{\boldsymbol{\beta=}0.04}$}  & \textbf{169.218 $\pm$ 1.636} \\
Optimistic $\beta=0.01$  & 168.579 $\pm$ 2.013 \\
Optimistic $\beta=0.003$ & 168.079 $\pm$ 2.460 \\
Hysteretic Q-Learning \cite{matignon2007hysteretic} & 167.921 $\pm$ 1.994 \\
Decentralized Q-learning            & 167.123 $\pm$ 4.131 \\
\bottomrule
\end{tabular}
\end{table}

%Table~\ref{tab:comparison-results} summarizes the final performance of different methods in terms of mean return and standard deviation across 10 random seeds. Overall, the optimistic variants perform best, with the risk-sensitivity parameter $\beta=0.04$ achieving the highest mean return and the lowest standard deviation. Slight deviations in $\beta$ (e.g., $\beta=0.08$ or $\beta=0.02$) result in only marginally lower performance, indicating that optimistic updates are robust to parameter choices, while careful tuning can further improve results. The hysteretic method achieves comparable but slightly lower performance, whereas the decentralized baseline yields stable yet lower returns. Notably, decentralized Q-learning exhibits higher variability, reflecting greater sensitivity to initialization and randomness; in some seeds it performs comparably, while in others it underperforms relative to hysteretic and optimistic methods (see Fig.~\ref{fig:curve_ball_balancing}). These results underscore the advantage of the optimistic strategy in enhancing both efficiency and stability during learning.

\begin{figure}[htbp]
    \vspace{-5pt}
    \centering
    \includegraphics[width=0.6\linewidth]{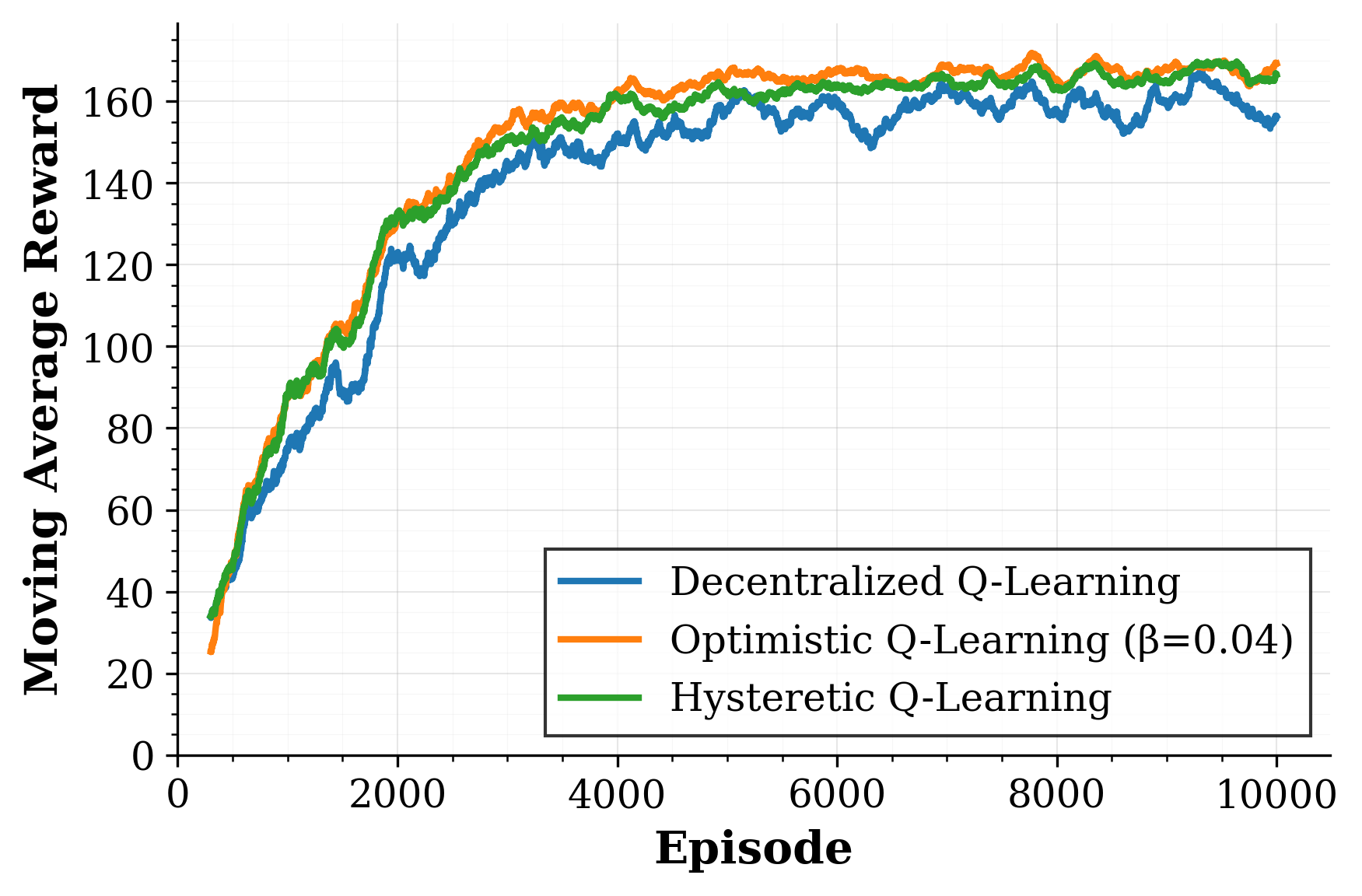}
    \caption{Learning curves for different algorithms.}
    \label{fig:curve_ball_balancing}
\vspace{-5pt}
\end{figure}

\section{Conclusion}
In this work, we revisited the role of optimism in MARL through a risk-sensitive lens.
By introducing \emph{optimistic value functions}, we provided a rigorous connection between  risk measures and optimism, framing optimism as a principled form of controlled risk-seeking.
This perspective led to a policy-gradient theorem for optimistic objectives and to decentralized actor-critic algorithms that implement these ideas in practice.

Our experiments demonstrate that risk-seeking optimistic updates consistently improve coordination compared to risk-neutral and heuristic baselines, highlighting that principled optimism yields both theoretical clarity and practical gains.

Looking ahead, several promising directions remain open. One is to extend our algorithm design to broader families of risk measures and to establish rigorous sample-complexity guarantees for the proposed optimistic algorithms. Another is to adapt our approach for integration with established MARL methods such as MAPPO and QMIX. Finally, applying our framework to high-stakes domains—including autonomous driving, energy management, and distributed robotics—would provide a compelling test of its ability to enhance cooperation in real-world multi-agent systems.
\vspace{-10pt}

\bibliographystyle{ieeetr}
\bibliography{bib}
\vspace{-10pt}
\appendix
\subsection{Proof of Theorem \ref{theorem:policy-gradient-optimistic-value-function}}\label{apdx:proof-policy-gradient}
\begin{proof}

    Taking differential on both side of \eqref{eq:optimistic-value-function-bellman} we get
    \begin{talign*}
        &\quad \nabla_\theta V^{\pi^\theta}(s) = \frac{\partial \sigma(\pi_s^\theta, -Q^{\pi^\theta}(s,\cdot))}{\partial \theta}\\
        % &= \sum_{a} \frac{\partial \sigma(\pi_s^\theta, -Q^{\pi^\theta}(s,\cdot))}{\partial \pi^\theta(a|s)}\nabla_\theta\pi^\theta(a|s) \\
        % &\qquad + \sum_{a} \frac{\partial \sigma(\pi_s^\theta, -Q^{\pi^\theta}(s,\cdot))}{\partial Q^{\pi^\theta}(s,a)}\nabla_\theta Q^{\pi^\theta}(s,a)\\
        &= \sum_{a} \frac{\partial \sigma(\pi_s^\theta, -Q^{\pi^\theta}(s,\cdot))}{\partial \pi^\theta(a|s)}\pi^\theta(a|s)\nabla_\theta\log\pi^\theta(a|s) \\
        &\qquad +  \sum_{a} \frac{\partial \sigma(\pi_s^\theta, -Q^{\pi^\theta}(s,\cdot))}{\partial Q^{\pi^\theta}(s,a)}\nabla_\theta Q^{\pi^\theta}(s,a).
    \end{talign*}
    Further, from Lemma 6 in \cite{zhang2023softrobust}, we have that $\frac{\partial \sigma(\pi_s^\theta, -Q^{\pi^\theta}(s,\cdot))}{\partial Q^{\pi^\theta}(s,a)} = \hpi^\theta(a|s)$, where $\hpi$ is defined as in \eqref{eq:def-hpi}. Substituting this equation to the above equations we get
    \begin{talign*}
     \nabla_\theta V^{\pi^\theta}(s) = \sum_{a} \frac{\partial \sigma(\pi_s^\theta, -Q^{\pi^\theta}(s,\cdot))}{\partial \pi^\theta(a|s)}\pi^\theta(a|s)\nabla_\theta\log\pi^\theta(a|s)\\
     + \sum_a \hpi^\theta(a|s)\nabla_\theta Q^{\pi^\theta}(s,a).
    \end{talign*}
    \begin{talign*}
       \textup{Since}~ \nabla_\theta Q^{\pi^\theta}(s,a) 
       % &\!=\! \textstyle \nabla_\theta\!\left(r(s,a) \!+\! \gamma\sum_{s'}P(s'|s,a)V^{\pi^\theta}(s')\right) \\
        &\textstyle =\gamma \sum_{s'}P(s'|s,a) \nabla_\theta V^{\pi^\theta}(s'),
    \end{talign*}
    we have
    % \begin{talign*}
    %     \nabla_\theta V^{\pi^\theta}(s) = \sum_{a} \frac{\partial \sigma(\pi_s^\theta, -Q^{\pi^\theta}(s,\cdot))}{\partial \pi^\theta(a|s)}\pi^\theta(a|s)\nabla_\theta\log\pi^\theta(a|s) \\+\gamma \sum_{s',a} \hpi^\theta(a|s)P(s'|s,a) \nabla V^{\pi^\theta}(s'),
    % \end{talign*}
    % i.e.,
    \begin{talign*}
        \nabla_\theta V^{\pi^\theta}(s_0) \!=\! \sum_{a} \!\frac{\partial \sigma(\pi_{s_0}^\theta, -Q^{\pi^\theta}(s_0,\cdot))}{\partial \pi^\theta(a|s_0)}\pi^\theta(a|s_0)\nabla_{\!\theta}\log\pi^\theta(a|s_0) \\+\gamma \bE_{s_1\sim P(\cdot|s_0,a_0), a_0\sim\hpi^\theta(\cdot|s_0)}\nabla V^{\pi^\theta}(s_1).
    \end{talign*}
    Applying this equation iteratively we get
    \begin{small}
    \begin{talign}
        &\nabla_\theta V^{\pi^\theta}(s_0) =\notag\\
        &\sum_{t\!=\!0}^{\!+\!\infty} \!\gamma^t \bE_{s_t, a_t\sim \hpi^\theta\!, P} \!
       \sum_{a} \!\frac{\partial \sigma(\pi_{s_t}^\theta, -Q^{\pi^\theta}(s_t,\cdot))}{\partial \pi^\theta(a|s_t)}\pi^\theta(a|s_t)\nabla_\theta\!\log\pi^\theta(a|s_0)\notag\notag\\
        &= \frac{1}{1-\gamma}\sum_{s} d^{\hpi^\theta}_{s_0}(s) \sum_{a} \frac{\partial \sigma(\pi_s^\theta, -Q^{\pi^\theta}(s,\cdot))}{\partial \pi^\theta(a|s)}\pi^\theta(a|s)\nabla_\theta\log\pi^\theta(a|s)\label{eq:PG-theorem-proof}
    \end{talign}
    \end{small}
   \noindent Specifically, when $\sigma(\pi^{\theta}_s, \!-\!Q^{\pi^{\!\theta}}\!(s,\!\cdot)) \!=\!  \beta^{\!-\!1}\!\log\!\bE_{a\sim \pi^{\theta}_s}e^{\beta Q^{\pi^{\!\theta}}\!\!(s,a)}$,
    we have $ \hpi^\theta(a|s) \propto \pi^\theta(a|s) e^{\beta Q^{\pi^\theta}(s,a)}$, and that
    \begin{talign*}
        \frac{\partial \sigma(\pi_s^\theta, -Q^{\pi^\theta}(s,\cdot))}{\partial \pi^\theta(a|s)} = \beta^{-1}\frac{e^{\beta Q^{\pi^\theta}(s,a)}}{\bE_{a\sim \pi^{\theta}_s}e^{\beta Q^{\pi^\theta}(s,a)}}.
    \end{talign*}
    Additionally, since 
    \begin{talign*}
        &V^{\pi^\theta}(s) = \sigma(\pi^{\theta}_s, -Q^{\pi^\theta}(s,\cdot)) =  \beta^{-1}\log\bE_{a\sim \pi^{\theta}_s}e^{\beta Q^{\pi^\theta}(s,a)} \\
        &\Longrightarrow~~
        \bE_{a\sim \pi^{\theta}_s}e^{\beta Q^{\pi^\theta}(s,a)} = e^{\beta V^{\pi^\theta}(s)}.
    \end{talign*}
    \begin{talign*}
       \textup{Thus}~~ \frac{\partial \sigma(\pi_s^\theta, -Q^{\pi^\theta}(s,\cdot))}{\partial \pi^\theta(a|s)} = \beta^{-1}\frac{e^{\beta Q^{\pi^\theta}(s,a)}}{e^{\beta V^{\pi^\theta}(s)}} = \beta^{-1} e^{\beta A^{\pi^\theta}(s,a)}.
    \end{talign*}
    Substitute this into \eqref{eq:PG-theorem-proof} we get
    \begin{small}    
    \begin{talign*}
        \nabla_\theta V^{\pi^\theta}(s_0)\!=\!\frac{1}{\beta(1\!-\!\gamma)}\sum_{s,a} d^{\hpi^\theta}_{s_0}(s) e^{\beta A^{\pi^\theta}(s,a)}\pi^\theta(a|s)\nabla_{\!\theta}\log\pi^\theta(a|s),
    \end{talign*}
    \end{small}
which completes the proof.
\end{proof}

\subsection{Equilibrium Convergence}\label{sec:equilibrium_convergence}
This section establishes the connection between Nash equilibria and first-order stationary points. 
We begin by defining the first-order stationary point and Nash equilibrium (NE).
\begin{defi}\label{def:first-order}{(First-order stationary policy)}
A set of policy parameters $\theta$ is called a first-order stationary policy if $(\theta' -\theta)^\top \nabla_{\theta} \bE_{s\sim\rho} V^{\pi^\theta}(s)\le 0$ for all possible $\theta'$.
\end{defi}
Note that first stationary points can be easily found by policy gradient type of methods.
\begin{defi}\label{def:NE}{(Nash Equilibrium)}
    For decentralized policy class (Eq. 1), where agent $i$'s policy is denoted as $\pi_i$, then policy $\pi = \{\pi_i\}_{i=1}^n$ is called a Nash equilibrium if
    \begin{align*}
        \bE_{s\sim\rho}\left(V^{\pi_i, \pi_{-i}}(s)\right)\ge V^{\pi_i', \pi_{-i}}(s), ~~\forall~\pi_i',~ i\in \{1,2,\dots,n\},
    \end{align*}
\end{defi}
i.e., no agent can increase the value function by unilaterally deviating from the current policy. Here $\pi_{-i}$ denotes the concatenation of all other agents policy except for $i$.

Given the definitions, we now show that these two notions are equivalent under mild assumptions:
\begin{theorem}
    Assume that the initial condition satisfies $\rho(s) >0,~\forall s$, and consider the setting with decentralized policy and direct parameterization, i.e., the policy parameter is given by $\theta = (\theta_1, \theta_2, \dots, \theta_n)$ where $\theta_i$ takes the form of Eq. (2) in the paper, and the policy is given by $\pi^\theta = \{\pi_i^{\theta_i}\}_{i=1}^n$, then for any $\pi^\theta$ that is a determinisitic policy, then $\theta$ is a first-order stationary point is equivalent to $\pi^\theta$ is a Nash equilibrium.
\begin{proof}
    Note that for deterministic policy $\pi$,
    \begin{align*}
        V^\pi(s) \!&=\! \max_{\hpi}\! \left[\!\bE_{s_t, a_t\sim \hpi}\sum_{t=0}^{+\infty}\gamma^t \left(r(s_t, a_t) \!-\! D(\hpi_{t, s_t}||\pi_{s_t})\right)\Big|s_0\!=\!s\right]\\
        &= \left[\bE_{s_t, a_t\sim \pi}\sum_{t=0}^{+\infty}\gamma^t r(s_t, a_t)\Big|s_0=s\right],
    \end{align*}
    because in this case the auxiliary optimistic policy $\hpi = \pi$.

    It is relatively easy to show that a Nash equilibrium is a first-order stationary point:
    \begin{align*}
        V^{(1-\eta)\pi_i^\theta\!+\!\eta (\pi_i^{\theta'}, \pi_{-i}^\theta)}(s) \!-\! V^{\pi_i^\theta, \pi_{-i}^\theta} \!=\! \eta(\theta_i \!-\! \theta_i')^\top \nabla_{\theta_i} V^{\pi^\theta}(s) \!+\! o(\eta),
    \end{align*}
    and let $\eta\to 0$ we prove the claim that a pure strategy Nash equilibrium is a first-order stationary point.

    We now prove that when a deterministic policy $\pi^\theta$ is a first-order stationary point, it is also a NE:

    Note that for deterministic policy $A^{\pi^\theta}(s,\pi^\theta(s))=0,$ and thus
    \begin{align*}
        \oA_i^{\pi^\theta}(s,\pi^{\theta_i}_i(s)) =  \beta^{-1} e^{\beta A^{\pi^\theta}(s,\pi^\theta(s))} = \beta^{-1}
    \end{align*}
    From the fact that $\theta$ is a first order stationary point, by definition we have that
    \begin{align*}
        &\sum_{a_i}(\hpi_i(a_i|s) - \pi^{\theta_i}(a_i|s))\nabla_{\theta_{s,a_i}}\bE_{s_0\sim\rho}V^{\pi^\theta}(s_0)\le 0, \forall \hpi_i\\
        \Longrightarrow~ &d^{\pi^\theta}(s)\sum_{a_i}(\hpi_i(a_i|s) - \pi^{\theta_i}(a_i|s))\oA_i^{\pi^\theta}(s,a_i)\le 0, \forall \hpi_i\\
        \Longrightarrow~ &\sum_{a_i}\hpi_i(a_i|s)e^{\beta A^{\pi^\theta}(s, a_i, \pi_{-i}^\theta(s))} \le 1, ~~\forall ~\hpi_i\\
        \Longrightarrow~& \forall a_i, A^{\pi^\theta}(s, a_i, \pi_{-i}^\theta(s))\le 0
    \end{align*}
    and thus from performance difference lemma (\cite{ZHANGgradient,kakade2002approximately}) we have
    \begin{align*}
    &\quad   \bE_{s_0\sim\rho} V^{\pi_i^{\theta_i'}, \pi_{-i}^{\theta_{-i}}}(s_0) - V^{\pi_i^{\theta_i}, \pi_{-i}^{\theta_{-i}}}(s_0) \\
      &\le \! \max_{\hpi_i} \bE_{\!\!\!s_t, a_t\sim (\hpi_i,\pi_{-i}^{\theta_{-i}})}\!\sum_{t=0}^{+\infty}\!\gamma^t r(s_t, a_t) \!-\! \bE_{s_t, a_t\sim \pi^\theta}\!\sum_{t=0}^{+\infty}\!\gamma^t r(s_t, a_t)\\
       & =\! \sum \!d^{\pi_i^{\theta_i'}, \pi_{-i}^{\theta_{-i}}}(s)\sum_{a_i}(\hpi_i(a_i|s)\! -\!\pi_i^{\theta_i}(a_i|s)) A^{\pi^\theta}\!(s, a_i, \pi_{-i}^\theta(s))\\
       &\le 0,
    \end{align*}    
    which proves that $\pi^\theta$ is a Nash equilibrium.
\end{proof}
\end{theorem}

\end{document}